%% file: main.tex
\documentclass{article} 
\usepackage{arxiv}
\usepackage[utf8]{inputenc}
\usepackage[T1]{fontenc}
\usepackage[authoryear,round]{natbib}
\renewcommand{\shorttitle}{Depth Transfer of Neural Optimizers}
\date{}
\usepackage{comment}
\input{math_commands.tex}

\usepackage{booktabs}
\usepackage[hidelinks]{hyperref}
\usepackage{url}
\usepackage{amsthm}
\usepackage{amssymb}
\usepackage{graphicx}
\usepackage{enumitem}
\usepackage{float}
\usepackage{mathtools}
\newtheorem{theorem}{Theorem}
\usepackage{amsthm}

\theoremstyle{definition}
\newtheorem{assumption}{Assumption}
\newtheorem{definition}{Definition}
\newtheorem*{theorem*}{Theorem}
\newtheorem{lemma}{Lemma}
\newtheorem{remark}{Remark}
\newtheorem{corollary}{Corollary}
\newtheorem{proposition}{Proposition}
\usepackage{amsmath}
\newcommand{\Hnorm}[1]{\left\|#1\right\|_H}

\newcommand{\Pnorm}[1]{\left\|#1\right\|_P}
\newcommand{\Hnormstar}[1]{\left\|#1\right\|_{H^*}}
\newcommand{\Pnormstar}[1]{\left\|#1\right\|_{P^*}}
\newcommand{\Hopnorm}[1]{\left\|#1\right\|_{H\to H}}

\newcommand{\PHopnorm}[1]{\left\|#1\right\|_{P\to H}}

\usepackage{cleveref}
\crefname{assumption}{Assumption}{Assumptions}
\Crefname{assumption}{Assumption}{Assumptions}
\crefname{corollary}{Corollary}{Corollaries}
\Crefname{corollary}{Corollary}{Corollaries}
\crefname{lemma}{Lemma}{Lemmas}
\Crefname{lemma}{Lemma}{Lemmas}
\crefname{theorem}{Theorem}{Theorems}
\Crefname{theorem}{Theorem}{Theorems}
\crefname{proposition}{Proposition}{Propositions}
\Crefname{proposition}{Proposition}{Propositions}
\crefname{definition}{Definition}{Definitions}
\Crefname{definition}{Definition}{Definitions}
\crefname{remark}{Remark}{Remarks}
\Crefname{remark}{Remark}{Remarks}

\title{On the Depth Transfer of Neural Optimizers via $L$-Smoothness}

\author{
  Maximilian Wong\\
  University of Chicago \\
  \texttt{maxwong7@uchicago.edu}
  \And
  Chenyu Xu\\
 Peking University\\
    \texttt{2400010840@stu.pku.edu.cn}
    \And
    Jiajin Li \\
Sauder School of Business \\
University of British Columbia \\
\texttt{jiajin.li@sauder.ubc.ca}
\And
Yiping Lu \\
Beijing International Center for Mathematical Research \\
Peking University \\
\texttt{yipinglu@bicmr.pku.edu.cn}
}

\begin{document}
\raggedbottom

\maketitle

\begin{abstract}
Scaling deep learning models efficiently requires predictable performance transfer across compute regimes. Establishing reliable hyperparameter transfer allows optimal training dynamics and structural choices to be discovered on low-cost proxy models, avoiding catastrophic failures during massive-scale training. Existing depth-transfer analyses are architecture-specific and heuristic: hyperparameter transfer is only justified through toy blocks, and both rely on infinite-depth dynamics, a specific initialization, and feature-learning criteria. Building on $L$-smoothness framework of \citet{xu2026moga}, we understand hyperparameter transfer as how do the Lipschitz and $L$-smoothness constants of the loss scale with the depth $K$? We show that inverse-depth residual branch scaling $\alpha=\Theta(1/K)$ is sufficient to make both constants depth-independent at fixed width, uniformly over the bounded parameter domain and for arbitrary $p\to q$ geometries. Because the bounds depend on each residual branch only through its Lipschitz and smoothness constants, the argument applies to general residual networks—including branches of block depth $\geq 2$ such as transformer and ResNet blocks—without assumptions on initialization or correlations between layer weights. Since the smoothness constant controls the curvature penalty in the descent
inequality, our bounds motivate a learning-rate scale without an additional
depth-dependent curvature correction for updates bounded in the prescribed
norm. This provides a geometric explanation of inverse-depth residual scaling
and its application to the operator-norm geometries associated with AdamW,
Muon, and MOGA. Learning-rate sweeps on GPT-2 (depths 2--128) and CIFAR-100 ResNets (8--32 residual blocks) confirm that \textrm{Muon} and \textrm{MOGA} exhibit stable learning-rate transfer under these prescriptions, with a common optimal rate across a 64-fold increase in depth. 
\end{abstract}
\section{Introduction}
Hyperparameter choices play a central role in the stability, efficiency, and performance of neural network training.  Finding suitable values often requires many training runs, making hyperparameter tuning increasingly expensive as models grow \citep{yang2022tensorprogramsv,yang2023spectral,yang2023depthmup,dey2025completep,xu2026moga,mlodozeniec2026completed,jiang2026hyperparameter,ma2026mu,shigida2026nugpt}. Moreover, settings that work well for a small model may become ineffective at a larger scale, limiting the value of inexpensive preliminary experiments. Understanding how hyperparameters should change with model size is therefore essential for translating insights from small experiments into reliable and efficient training of large models. In practice, many state-of-the-art optimizers show optimal learning rate drifts as depth scales up (Figure \ref{fig:mlp-depth-lr-transfer}). In this paper, we aim to answer : 
\begin{center}
    \emph{How can we design optimization algorithms so that optimal learning rates transfer reliably as network depth increases?}
\end{center}

Prior work has approached depthwise hyperparameter transfer through parameterizations that preserve feature learning as network depth increases. Depth $\mu$P \citep{yang2023depthmup}  analyze training dynamics as width tends to infinity followed by depth, identifying feature learning and feature diversity as central design criteria for residual networks with one layer per branch. CompleteP \citep{dey2025completep} introduce complete feature learning to motivate residual scaling inversely proportional to depth and corresponding learning rate prescriptions for architectures with multiple layers per branch, including transformers. These theoretical analyses rely on specific architectural and initialization assumptions, are tailored to particular optimizers, and primarily characterize feature behavior in asymptotic regimes. However, preserving feature learning does not by itself establish that the optimal learning rate remains unchanged across depths. This raises a broader question: can we establish a general theory of depthwise learning rate transfer that applies to a broader class of architectures and optimizers, without relying on specific random initialization schemes or infinite width and depth limits?

\begin{figure}[t]
    \centering    \includegraphics[width=\linewidth]{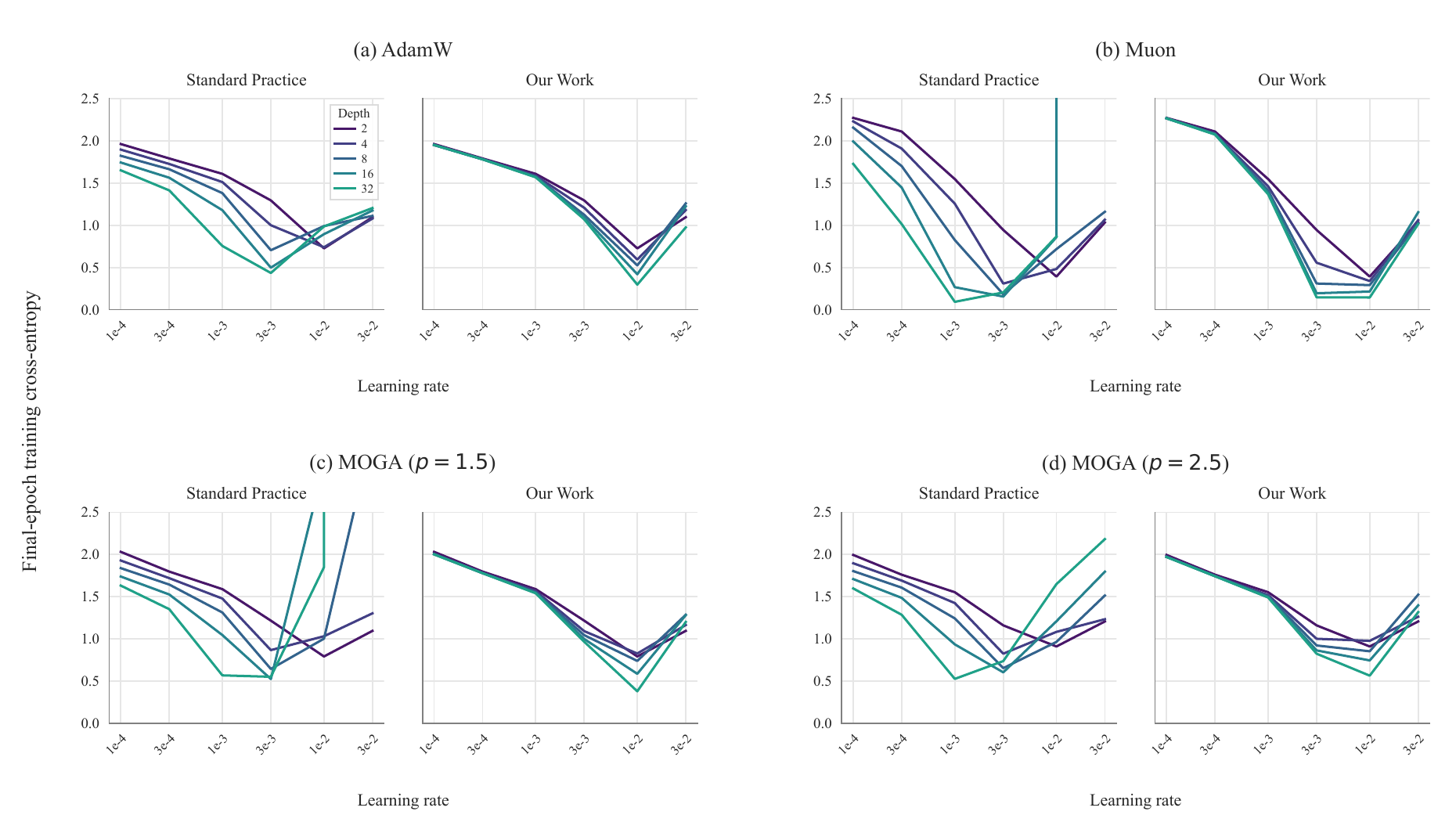}
    \vspace{-0.3in}
    \caption{Final-epoch training loss against learning rate for residual MLPs of varying depth trained on CIFAR-10 with AdamW, Muon, and MOGA. Each branch contains two linear layers, and width is fixed across depths. Under standard unscaled residual branches (left), the best learning rate varies with depth. Our inverse-depth scaling, ($\alpha=K_{\mathrm{base}}/K$) with ($K_{\mathrm{base}}=2$) (right), show that optimal learning rate is aligned across depths.}
    \label{fig:mlp-depth-lr-transfer}
\end{figure}

To understand depthwise learning-rate transfer, we follow \cite{xu2026moga} and characterize the stable learning rate through a scale-independent $L$-smoothness control. The $L$ smoothness condition controls how much the loss gradient changes as the network weights vary, measuring weight perturbations in a norm adapted to the optimizer and gradient changes in its dual. $L$-smoothness governs step-size selection: a gradient step is guaranteed to decrease the loss only when the learning rate is of order $1/L$, and larger steps overshoot the region where the linearization remains valid. Consequently, the stable learning rate scales inversely with $L$. The obstacle to a uniform depth bound $L$ is that the influence of an optimization step is not local. A full-network gradient update needs to perturb all layers simultaneously,  requiring control of both perturbation propagation and second-order interactions between layers. \textbf{Our framework extends the depth scaling analysis beyond the Adam prescriptions developed in Depth $\mu$P and CompleteP to a broader class of optimizers used in LLM pretraining, including Muon and row normalization.} Our bounds apply directly to finite networks and do not require the sequential limit used in Depth $\mu$P and CompleteP, where the width is taken to infinity before depth.



Numerical results support our theoretical observations. We evaluate learning-rate transfer under this prescription through sweeps on GPT-2 language models and CIFAR-100 ResNets. The training loss sweeps show an aligned low-loss learning-rate region for Muon and MOGA across GPT-2 depths from 2 to 128. On CIFAR-100, all three evaluated MOGA settings share the same best tested learning rate across ResNets with 8, 16, and 32 residual blocks, measured by final-epoch training loss. AdamW also retains a common low-loss region, although its best tested rate shifts toward smaller values at the largest GPT-2 depths. Together, the theoretical bounds and empirical results provide a geometric account of depth scaling and support learning-rate transfer across the evaluated architectures and optimizer geometries.
\section{Preliminaries}
\subsection{Settings, Notations and Assumptions}
We consider the optimization problem of minimizing a neural network loss function:
\begin{align*}
    \min_{\theta}\mathcal L(\theta):= l(y_K(x;\theta))
\end{align*}

Let \((H,\Hnorm{\cdot})\) be the hidden-state space and let \((P,\Pnorm{\cdot})\) be the parameter space. Then for all \(k\), we have \(y_k\in H\), \(W_k,U_k,V_k\in P\), and \(f_k:H\times P\to H\). Throughout this paper, we assume that the branch function \(f_k\) is twice continuously Fréchet differentiable. The derivatives are maps \(\partial_y f_k:H\to H\) and \(\partial_{W_k} f_k:P\to H\). The second derivatives are maps \(\partial_y^2 f_k:H\times H\to H\), \(\partial_y\partial_{W_k} f_k:H\times P\to H\), and \(\partial^2_{W_k} f_k:P\times P\to H\).

For a parameter tuple \(W=(W_1,\ldots,W_K)\), define the block norm
\begin{align*}
\|W\|_{\mathrm{block}} = \max_{1\le k\le K} \Pnorm{W_k},
\end{align*}
and its dual
\begin{align*}
\|G\|_{\mathrm{block},*} = \sum_{k=1}^{K} \Pnormstar{G_k},
\end{align*}
where \(\Pnormstar{\cdot}\) denotes the dual norm on \(P^*\).

For notational simplicity, we write \(f_{W_k}(y):=f_k(y,W_k)\).

\begin{definition}[Abstract Residual Network]\label{def:resnet}
Let \((H,\|\cdot\|_H)\) and \((P,\|\cdot\|_P)\) be normed vector spaces. Let \(x\in H\) be the input. Define \(y_0 = x\) and for \(k=1,\ldots,K\),
\begin{align*}
y_k = y_{k-1} + \alpha f_{W_k}(y_{k-1}),
\end{align*}
where \(f_{W_k}:H\to H\) is a twice continuously differentiable residual function parameterized by \(W_k\in P\). In this paper, we set the scaling factor as
\begin{align*}
\alpha = \frac{\beta}{K}, \qquad \beta = \Theta(1).
\end{align*}
The network output is \(y_K\in H\), and the loss is
\begin{align*}
\mathcal{L}(W) = l(y_K),
\end{align*}
with \(l:H\to\mathbb{R}\). The parameter tuple is \(W=(W_1,\ldots,W_K)\), and the perturbation directions are \(\Delta U=(\Delta U_1,\ldots,\Delta U_K)\) and \(\Delta V=(\Delta V_1,\ldots,\Delta V_K)\), with \(\Delta U_k,\Delta V_k\in P\).
\end{definition}

We make the following regularity assumptions on the network we defined. They are imposed throughout the paper.
\begin{assumption}[Initial Condition and Perturbation Bounds]\label{as1}
\begin{align*}
    \Hnorm{x}\le C_x,\qquad
    \Pnorm{\Delta U_k}\le 1,\qquad
    \Pnorm{\Delta V_k}\le 1.
\end{align*}
\end{assumption}
\begin{assumption}[Residual Function Regularity]\label{as2}
There exist positive constants \(C_f\), \(L_f\), \(M_f\),
\(C_W\), \(M_{Wf}\), and \(M_{WW}\), independent of \(K\),
such that the following bounds hold for every \(k=1,\ldots,K\),
every \(y,y_1,y_2\in H\), every \(W_k,W_k'\in P\),
and every \(\Delta U_k,\Delta V_k\in P\). In appendix B.1 and B.2, with mild assumptions on parameter space and function regularity, we can satisfy our abstract residual function regularity for MLP and Attention blocks
\begin{itemize}
    \item[2.1]\textbf{First Order Regularity in Hidden State}
    \begin{align*}
        \Hnorm{f_{W_k}(y_1)-f_{W_k}(y_2)} \le L_f \Hnorm{y_1-y_2}
    \end{align*}
    \item[2.2]\textbf{Second Order Regularity in the Hidden State}
    \begin{align*}
        \Hopnorm{\partial_y f_{W_k}(y_1)-\partial_y f_{W_k}(y_2)} \le M_f \Hnorm{y_1-y_2}
    \end{align*}
    \item[2.3] \textbf{First Order Regularity in the Parameters}
    \begin{align*}
        \Hnorm{f_{W_k}(y)-f_{W_k'}(y)} \le C_W \Pnorm{W_k-W_k'}
    \end{align*}
    \item[2.4] \textbf{Mixed-State Parameter Regularity}
    \begin{align*}
        \Hopnorm{\partial_y f_{W_k}(y)-\partial_y f_{W_k'}(y)} \le M_{Wf} \Pnorm{W_k-W_k'}
    \end{align*}
    \item[2.5] \textbf{Second Order Regularity in the Parameters}
    \begin{align*}
        \Hnorm{\partial_{W_k}f_{W_k}(y)[\Delta U_k]-\partial_{W_k'}f_{W_k'}(y)[\Delta U_k]} \le M_{WW} \Pnorm{W_k-W_k'} \Pnorm{\Delta U_k}
    \end{align*}
    \item[2.6] \textbf{Bounded Output at Initial State}
    \begin{align*}
        \Hnorm{f_{W_k}(y_0)} \le C_f
    \end{align*}
\end{itemize}
\end{assumption}
\begin{remark} Following \citep{xu2026moga}, we show that our assumption is realistic. Their analysis was conducted on MLP and attention blocks, and they imposed activation function regularity and bounded parameter assumptions that are consistent with our generalized abstract-residual function \cref{as2} for the corresponding MLP and Attention blocks. Our set-up is therefore a faithful abstraction of the standard conditions satisfied by modern architectures; the detailed reductions are carried out in the appendix via \cref{prop:mlp} and \cref{prop:attention}.
\end{remark}
\begin{assumption}[ Regularity of the Loss Function]\label{as3}
There exist constants \(L_J,M_J>0\), independent of \(K\), such that for all \(z\in H\),
\begin{align*}
\Hnormstar{\nabla l(z)} \le L_J, \qquad
|\nabla^2 l(z)[u,v]| \le M_J \Hnorm{u} \Hnorm{v}.
\end{align*}
\end{assumption}

\subsection{Related Works}

\paragraph{Optimizers as Steepest Descent Under Matrix Operator Norms} Many state-of-the-art optimizers can be approximated by steepest descent under some matrix operator norm $\|\cdot\|_{p\rightarrow q}$ \citep{bernstein2024oldoptimizer,pethick2025scion,xu2026moga}. The steepest descent direction $D^*$ is given by
\begin{align*}
    D^*=\argmin_{\|D\|_{p\rightarrow q}\leq 1} \langle \nabla f(\theta),D\rangle, \text{ where } \|D\|_{p\rightarrow q} :=\sup_{x\neq 0}\frac{\|Dx\|_q}{\|x\|_p} = \sup_{\|x\|_p=1}\|Dx\|_q
\end{align*}
For example, we will recover Adam \citep{kingma2015adam,balles2018dissecting}, SignSGD \citep{bernstein2018signsgd} and Lion \citep{chen2023symbolic} as steepest descent under $\ell_1\rightarrow \ell_\infty$, MUON as steepest descent under $\ell_2\rightarrow\ell_2$ \citep{jordan2024muon,bernstein2024oldoptimizer} and Row normalization as steepest descent under  $\ell_2\rightarrow\ell_\infty$ \citep{wen2025sron,yuan2026nora,deng2026rmnp,xu2026moga} .

Steepest descent under matrix operator norms is well motivated because these norms capture
the smoothness structure of neural networks: with respect to a suitable 
operator norm on each layer, the network output and loss are Lipschitz and
smooth in the weights with constants that remain controlled as the architecture scales \citep{xu2026moga}, so
the resulting steepest-descent update, obtained by dualizing the gradient under this norm,
takes steps adapted to the true local geometry of the loss \citep{large2024scalable,bernstein2024modularduality,xu2026moga}.

\paragraph{Hyperparameter Transfer.} Hyperparameter transfer, specifically learning-rate transfer, is the technique of reusing a learning rate tuned on a smaller network for the training of a larger model which eliminates the cost of redundant hyperparameter tuning. In particular, we study learning rate transfer over depth. Depth-$\mu$P \citep{yang2023depthmup} uses the tensor programs framework \cite{yang2020tensor,yang2022tensorprogramsv} to study the infinite-depth limit of training dynamics and derive parameterization rules that satisfy the activation and gradient scales feature learning rules as depth varies. However, depth-$\mu$P holds only for block depth 1 models and their scaling rules deteriorate for the commonly used transformer, which has block depth 2. To fix the limitation of Depth-$\mu$P, Complete-P \citep{dey2025completep} takes the $\mu$P framework and derives another set of rules targeted at block-depth $\geq 2$ architectures. Both works are limited to block-depth constraints and rely heavily on infinite-depth initialization and early training dynamics. Furthermore, rigorous convergence analysis of training dynamics are only established for block-depth 1, meaning that feature learning analysis cannot solely establish the connection to how the network evolves under training for block-depth $\geq2$ architecture. In contrast, our analysis is robust to block-depth; in particular, we provide a worst-case guarantee that is sufficient for architectures of any block-depth. Our analysis also directly controls the gradient scale and loss sensitivity of the network as depth increases which provides a more principled alternative explanation to learning rate transfer scaling rules for transformers. Lastly, our approach does not require specific initialization configurations or weight correlation assumptions, providing learning-rate scaling rules beyond specialized cases in prior works. 

Zheng et al. (2026) develop spectral conditions for joint width-depth scaling
across residual architectures and optimizers, including Muon. Their
parameterizations include inverse-depth residual scaling and learning rates
without a depth correction at fixed width. Our analysis provides a
complementary, unified framework through depth-uniform bounds on loss
sensitivity and curvature under prescribed parameter norms. By expressing
the bounds in terms of common residual-function regularity conditions,
the same argument applies across architectures and norm geometries
satisfying these conditions.

\paragraph{Norm Geometry and L-smoothness.} \citet{large2024scalable} establish Lipschitz and smoothness
guarantees in a modular norm constructed recursively from module
sensitivities and mass allocations, including depth-independent
bounds for a suitably scaled residual construction
(Propositions 4--5). We study depth scaling directly in a prescribed
parameter geometry. For the identity-skip recurrence
$y_k=y_{k-1}+(\beta/K)f_k(y_{k-1})$, we establish depth-uniform
Lipschitz and smoothness bounds at fixed width in the unweighted
maximum of prescribed parameter-block norms, under uniform branch
and loss regularity assumptions. Specializing these norms to
$p\to q$ operator norms makes the associated unit-ball
steepest-descent directions explicit and computable blockwise,
without introducing separate module-mass parameters.
The resulting smoothness bound gives a depth-independent curvature
penalty for updates measured in this geometry, providing the basis
for the learning-rate discussion in Section 3.2 and our
optimizer-specific prescriptions. The analysis treats each residual
branch through its regularity constants and therefore also covers
multi-layer branches satisfying these assumptions.

\section{Main Theorems}
\subsection{Depth-Independent Lipschitz Constant}
\cite{xu2026moga} showed that Lipschitz continuity can characterize the sensitivity of network loss to parameter perturbations and established a width-independent Lipschitz bound using mean-normalized norm, making the network stable in width. Here, we study whether an appropriate branch scaling can yield a depth-independent Lipschitz bound to make the loss function sensitivity controlled in depth. Denote $\delta y_a=\nabla_{W_a}y_i[\Delta U_a]$ the first order effect of perturbing $W_a$ in the direction $\Delta U_a$. For layers after $a$, the residual Jacobian $J_i=\partial y_i/\partial y_{i-1}$ maps the incoming perturbation to the next layers. The chain rule gives:
\begin{align*}
    \delta y_i = \delta y_{i-1}+\alpha \partial_yf_{W_i}(y_{i-1})[\delta y_{i-1}]=J_i \delta y_{i-1}
\end{align*}
Therefore for a later layer $b$ we have
\begin{align*}
    \delta y_b=J_b\dots J_{a+1}\delta y_a
\end{align*}
Under our assumptions, we have the following upper bound for the Jacobians (see the Appendix for details)
\begin{align*}
    \Hopnorm{J_i}\leq 1+C\alpha 
\end{align*}
Where $C\in O(1)$. Therefore, we have
\begin{align*}
    \|\delta y_b\|_H\leq (1+C\alpha)^{b-a}\|\delta y_a\|_H\leq (1+C\alpha)^K\|\delta y_a\|_H
\end{align*}
Which shows that the worst case guarantee blows up as $K\rightarrow \infty$ (number of layers between $a$ and $b$ increases). However, an appropriate branch scaling can control this bound. Setting $\alpha =1/K$, we have
\begin{align*}
    \|\delta y_b\|_{H}\leq (1+\frac{C}{K})^{b-a}\|\delta y_a\|_{H}& \leq (1+\frac{C}{K})^K\|\delta y_a\|_{H} \\ &\leq \exp (C)\|\delta y_a\|_{H}
\end{align*}
\begin{theorem}[Depth-Independent Lipschitz Continuity]\label{thm:lipschitz}
Under \cref{as1,as2,as3} and \(\alpha=\beta/K\), \(\beta=\Theta(1)\), there exists \(L_{\mathrm{Lip}}=O(1)\), independent of \(K\), such that
\begin{align*}
\sup_{\|\Delta U\|_{\mathrm{block}}\le 1} |\nabla \mathcal{L}(W)[\Delta U]| \le L_{\mathrm{Lip}}.
\end{align*}
Consequently,
\begin{align*}
|\mathcal{L}(W^1)-\mathcal{L}(W^2)| \le L_{\mathrm{Lip}} \|W^1-W^2\|_{\mathrm{block}}.
\end{align*}
\end{theorem}

\subsection{Depth-Independent L-smoothness Theorem}
While Lipschitz continuity is enough to control the sensitivity of the network to weight perturbations, it provides limited information on the actual optimization dynamics during training. We show that inverse depth branch scaling provides a depth-independent L-smoothness constant that directly controls how much gradients can vary along the optimization trajectory. Bounding this quantity is crucial, as it ensures that the steps we take in optimization are stable as depth varies. 
\begin{theorem}[Depth-Independent \(L\)-Smoothness]\label{thm:smooth}
Under \cref{as1,as2,as3} and \(\alpha=\beta/K\), \(\beta=\Theta(1)\), there exists \(L=O(1)\), independent of \(K\), such that
\begin{align*}
\sup_{\|\Delta U\|_{\mathrm{block}}\le 1, \|\Delta V\|_{\mathrm{block}}\le 1} |\nabla^2 \mathcal{L}(W)[\Delta U,\Delta V]| \le L.
\end{align*}
Consequently,
\begin{align*}
\|\nabla \mathcal{L}(W^1)-\nabla \mathcal{L}(W^2)\|_{\mathrm{block},*} \le L \|W^1-W^2\|_{\mathrm{block}}.
\end{align*}
\end{theorem}

\paragraph{Why L-Smoothness Motivates Learning Rate Transfer} Following \citet{xu2026moga}, we use the smoothness constant $L$ to guide learning rate selection. For an update $x^+=x+\eta D$ along a unit norm steepest descent direction, the linear approximation predicts a loss decrease of $\eta\|\nabla f(x)\|_*$, while the curvature penalty is bounded by $L\eta^2/2$. Minimizing the resulting upper bound on the loss gives $\eta^*=\|\nabla f(x)\|_*/L$. A common smoothness bound across depths therefore ensures that the curvature term does not require smaller learning rates as the network becomes deeper. If the dual gradient norm also remains comparable across depths, this rule yields learning rates of comparable magnitude, providing a theoretical motivation for depthwise learning rate transfer.

\subsection{Proof Sketch}
We outline the idea behind \cref{thm:lipschitz} and \cref{thm:smooth}. All constants below
are independent of the depth $K$; the whole argument rests on a $1/K$ cancellation against a
$K$-fold summation.

\paragraph{(i) A stable forward pass.}
With $\alpha=\beta/K$, each residual step $y_k = y_{k-1}+\alpha f_{W_k}(y_{k-1})$ perturbs the
state by only $O(1/K)$. Since there are only $K$ steps, the states remain uniformly bounded,
$\Hnorm{y_k}\le C_y$ for all $k$ (\cref{cor:fwd}). Hence every quantity evaluated at the iterates
$y_{k-1}$ is bounded a priori, independently of $K$.

\paragraph{(ii) Controlled Jacobian products.}
The layer Jacobian $J_k = I + \alpha\,\partial_y f_{W_k}(y_{k-1})$ satisfies
$\Hopnorm{J_k}\le 1+\alpha L_f$, so arbitrary products obey
\[
\|J_K\cdots J_j\|_{H\to H}
\le (1+\beta L_f/K)^K \le e^{\beta L_f} = C_J = O(1)
\]
(\cref{cor:jacprod}). Perturbations thus don't blow up as they propagate through the
network.

\paragraph{(iii) First-order directional derivatives are $O(1/K)$ each.}
By the chain rule, $\nabla_W\mathcal{L}[\Delta U]=l'(y_K)\sum_j s_K^{(j,U_j)}$. Each term
$s_K^{(j,U_j)}$ originates from the single source $\alpha\,\partial_{W_j}f_{W_j}(y_{j-1})[\Delta U_j]$
at layer $j$ and is then propagated by $J_K\cdots J_{j+1}$, giving
\[
\Hnorm{s_K^{(j,U_j)}} \le C_J\cdot \alpha C_W = O(1/K).
\]
Summing over the $K$ layers, $\bigl\|\sum_j s_K^{(j,U_j)}\bigr\|\le C_s=O(1)$ (\cref{cor:s-general}).

\paragraph{(iv) Second-order directional derivatives: $O(1/K)$ diagonal, $O(1/K^2)$ off-diagonal.}
The Hessian splits into an outer-product part $l''(y_K)(\sum_j s_K^{(j)})(\sum_\ell s_K^{(\ell)})
=O(1)$ and a propagation part $l'(y_K)\sum_{j,\ell}h_K^{(j,\ell)}$. In the latter, the source carries
two derivative hits, hence an $\alpha^2$ factor:
\[
\Hnorm{h_K^{(j,j)}}=O(1/K),\qquad \Hnorm{h_K^{(j,\ell)}}=O(1/K^2)\quad (j\neq\ell).
\]
There are $K$ diagonal pairs and $K(K-1)$ off-diagonal pairs, so
\[
\Bigl\|\sum_{j,\ell}h_K^{(j,\ell)}\Bigr\|
\le K\cdot O(1/K) + K^2\cdot O(1/K^2) = O(1).
\]

\paragraph{(v) From directional bounds to Lipschitz/smoothness.}
Having shown $\sup_{\|\Delta U\|\le1}|\nabla\mathcal{L}[\Delta U]|$ and
$\sup_{\|\Delta U\|,\|\Delta V\|\le1}|\nabla^2\mathcal{L}[\Delta U,\Delta V]|$ are uniformly bounded,
\cref{lem:lipschitz} and \cref{lem:smooth} (via the fundamental theorem of calculus) upgrade these to
the Lipschitz and smoothness statements in the block norm, completing the proof of \cref{thm:lipschitz}
and \cref{thm:smooth}.

\section{Experiments}
We evaluate our scaling laws on MOGA, Muon, and Adam. To examine the robustness of MOGA, we run on the norm parameter $p\in\{1.5,2,2.5\}$. We examine large-scale pretraining tasks on GPT2 \citep{radford2019gpt2} and the image classification tasks on ResNet \citep{he2016resnet} to show that our scaling prescriptions are robust to architecture. 

\subsection{Learning Rate Transfer On Transformers}
We use the GPT2 Model \citep{radford2019gpt2} and equip it with the standard tokenizer of vocabulary size 50257, context-length 1024, and the OpenWebText Dataset \citep{gokaslan2019openwebtext}. Below are the detailed configurations:
\begin{enumerate}
    \item \textbf{Model} We run GPT2 with fixed 256 width and 4 attention heads and vary the depth of the model over $\{2,4,8,16,32,64,128\}$. This model is run with Muon, AdamW, and Moga $p\in \{1.5,2,2.5\}$.
    \item \textbf{Learning-rate Schedule and Training Configuration} We run GPT2 from scratch with 256 effective batch size and 1200 iterations, corresponding to a 315 Million token budget per run. We use cosine learning rate decay with $10\%$ warmup and $\text{min\_lr = 0.1\ peak\_lr}$ meaning that the learning rate will increase linearly to the peak learning rate and decay with a cosine schedule to $10\%$ of the peak learning rate at the end of training. We sweep across six learning rates $\{1\times 10^{-3},2\times 10^{-3},4\times 10^{-3},1\times 10^{-2},2\times 10^{-2},4\times 10^{-2}\}$. 
\end{enumerate}
\begin{figure}[t]
    \centering
    \includegraphics[width=0.8\linewidth]
        {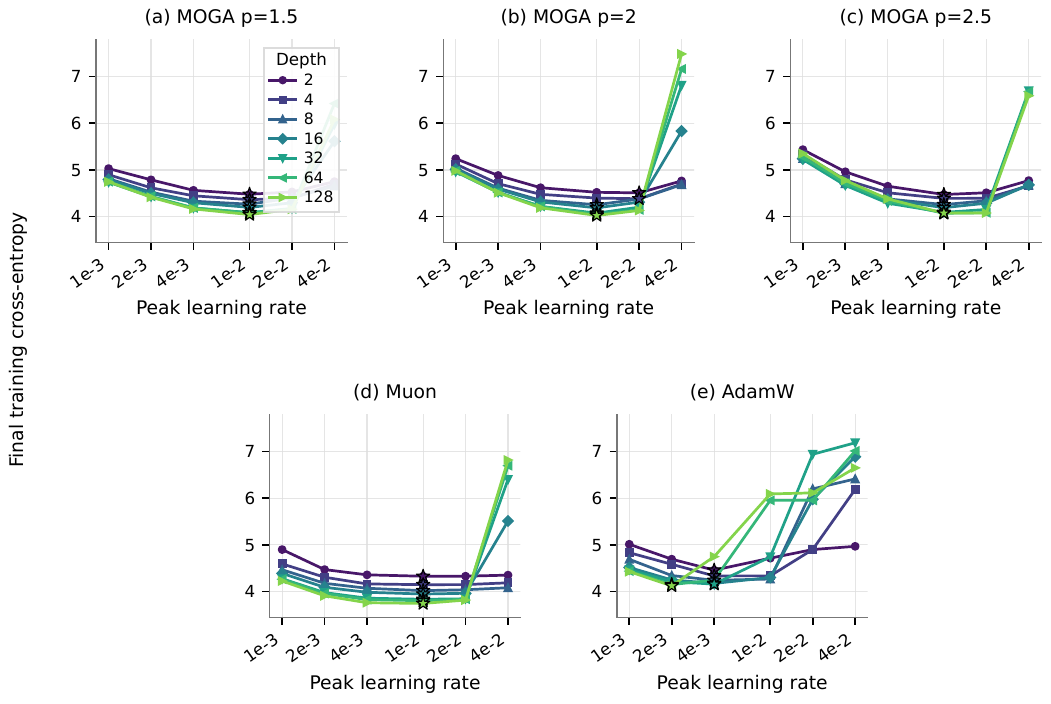}
    \caption{
        Learning-rate transfer across GPT-2 depths at fixed width.
        Panels show MOGA with $p\in\{1.5,2,2.5\}$, Muon, and AdamW.
        Colors and markers indicate transformer depth.
        Stars mark the lowest training loss among the tested
        learning rates at each depth.
    }
    \label{fig:transformer-lr-transfer-1}
\end{figure}
As shown in figure~\ref{fig:transformer-lr-transfer-1}, models spanning a 64-fold depth counts choose nearly the same optimal learning rate. This implies that the optimal learning rate is effectively invariant to model depth for MOGA and Muon. We observe that optimal learning rates tuned for shallow models can generally be transferred directly to deeper models, thus eliminating the cost of retuning hyperparameters. 

As our theory predicted, choosing different $p$ values for MOGA does not affect its learning rate transfer as none of the optimizer geometries require additional learning rate corrections. Of all the optimizers, Muon and MOGA exhibit a wide stable learning rate area that transfers as depth increases, indicating strong stable learning rate transfer. Notably, Adam becomes more unstable at large depths but the optimal learning rate shifts to the left only by a factor of $2$; which still shows effective learning rate transfer. 
\subsection{Learning Rate Transfer On ResNet}
We use the standard ResNet architecture \citep{he2016resnet} and equip it with the CIFAR100 dataset\citep{krizhevsky2009learning}. Below are the detailed configurations:
\begin{enumerate}
    \item \textbf{Model} We train on CIFAR-adapted ResNets with standard two-convolution BasicBlocks. We use stage configurations of $[2,2,2,2]$,$[3,4,6,3]$ and $[8,8,8,8]$ corresponding to ResNet18, ResNet34, and ResNet66 with depths of $K \in \{8,16,32\}$ residual blocks respectively. We train on the MOGA optimizer with $p\in \{1.5,2,2.5\}$. We train on CIFAR-100, splitting its 50,000 training examples into 45,000 training examples and 5,000 validation examples using a fixed split seed. Training images are augmented using random \(32\times32\) crops with four-pixel padding and random horizontal flipping. All images are normalized using the CIFAR-100 channel statistics.
    \item \textbf{Learning-rate Schedule and Training Configuration}. We train every model with batch size 64 for 20 epochs, corresponding to 703 steps per epoch and a total of 14060 updates. We use one epoch of linear warm up and cosine decay to $50\%$ of the peak learning rate. We sweep across six learning rates $\{4\times10^{-4},1\times 10^{-3},2\times 10^{-3},4\times 10^{-3},1\times 10^{-2},2\times 10^{-2}\}$
\end{enumerate}
\begin{figure}[t]
    \centering
    \includegraphics[width=0.9\linewidth]
        {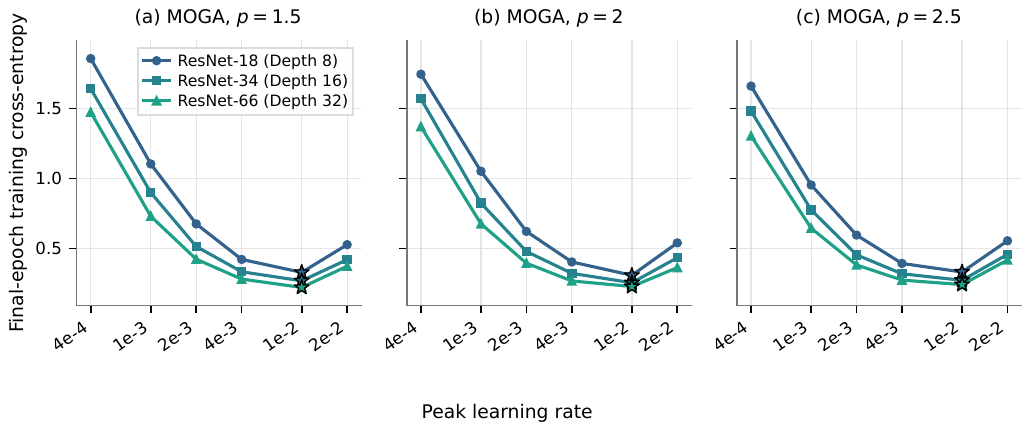}
    \caption{
        Learning-rate transfer on CIFAR-100 with depth-scaled ResNets.
        Panels show MOGA with $p\in\{1.5,2,2.5\}$.
        Curves report final-epoch training cross-entropy for
        ResNet-18, ResNet-34, and ResNet-66, corresponding to
        $K\in\{8,16,32\}$ residual blocks.
        Stars indicate the lowest loss among the tested
        learning rates for each model.
    }
    \label{fig:resnet-train-transfer-1}
\end{figure}

In figure~\ref{fig:resnet-train-transfer-1}, the three selected models select the same optimal learning rate over all 3 variants of the MOGA optimizer. An optimal learning rate tuned on ResNet-18 can be directly transferred to ResNet-66. Another particularly strong result is that all three variants showed matching stable learning rate areas that are also robust in depth. Again, these results support our claim that MOGA consistently transfers learning rates for different choices of $p$. These results also provide empirical evidence that our scaling prescriptions work for different architectures; in total we have tested GPT-2, ResNet, and residual MLPs for our baseline.

\section{Conclusion}
In this paper, we adopt the matrix-operator norm steepest descent framework from \citet{xu2026moga} to investigate how optimizer geometry can yield depthwise learning rate scaling rules. We show that many state-of-the-art optimizers- such as AdamW and MUON- can be studied through associated matrix-operator norm geometries. However, our analysis shows that neural networks are not depth-independent under the standard branch scaling; the worst case Lipschitz bound becomes unstable as network depth increases. 

We propose inverse-depth branch scaling to bound aggregate sensitivity to layer updates uniformly in depth and hence induce an L-Lipschitz constant that is independent of depth. Our further analysis of gradient sensitivity shows that the objective loss function under inverse-depth branch scaling sufficiently induces a depth-independent L-smooth constant which motivates a learning rate scale without an additional depth-dependent curvature correction. Our experiments evaluate the resulting learning-rate
transfer for practical optimizers.

Our experiments of large-scale pre-training tasks on GPT-2 and image classification tasks on ResNet show that optimal learning rates stay nearly invariant on the tested state-of-the-art optimizers, which allows models over a 64-fold depth range to effectively transfer learning rates from small models to deeper models, reducing the computation budget of hyperparameter tuning.

Our analysis establishes inverse-depth residual scaling as a worst-case guarantee for depth-uniform sensitivity and smoothness bounds. Larger residual multipliers can also support learning-rate transfer in certain settings. For example, Depth-\ensuremath{\mu}P \citep{yang2023depthmup} uses ($K^{-1/2}$) residual scaling and demonstrates transfer in residual networks with single-layer branches under its parameterization. Our present bounds do not explain this behavior. Extending the analysis to account for the architectural or statistical structure that enables transfer in these special settings is a natural direction for future work.
\newpage
\clearpage
\nocite{*}
\bibliographystyle{cit_style}
\bibliography{references}
\newpage
\appendix
\section{Proof of the Main Theorem}
\subsection{Corollaries}
\begin{corollary}[Forward Boundedness]\label{cor:fwd}
For all \(k=0,\ldots,K\),
\begin{align*}
\Hnorm{y_k} \le C_y := e^{\beta L_f} \Bigl( C_x + \frac{L_f C_x + C_f}{L_f} \Bigr).
\end{align*}
\end{corollary}
\begin{proof}
By \cref{as1} and \cref{as2},
\begin{align*}
\Hnorm{f_{W_k}(y_{k-1})}
&\le \Hnorm{f_{W_k}(y_{k-1})-f_{W_k}(y_0)} + \Hnorm{f_{W_k}(y_0)}
\le L_f \Hnorm{y_{k-1}-y_0} + C_f \\
&\le L_f \Hnorm{y_{k-1}} + L_f C_x + C_f.
\end{align*}
Let \(u_k := \Hnorm{y_k}\). Then
\begin{align*}
u_k \le (1+\alpha L_f) u_{k-1} + \alpha (L_f C_x + C_f).
\end{align*}
Unrolling,
\begin{align*}
u_k \le (1+\alpha L_f)^k u_0 + \alpha (L_f C_x + C_f) \sum_{i=0}^{k-1} (1+\alpha L_f)^i.
\end{align*}
Since \(u_0 \le C_x\) and \((1+\alpha L_f)^k \le e^{\beta L_f}\),
\begin{align*}
u_k \le e^{\beta L_f} \Bigl( C_x + \frac{L_f C_x + C_f}{L_f} \Bigr).
\end{align*}
\end{proof}

\begin{corollary}[Jacobian Bound]\label{cor:jac}
For all \(k=1,\ldots,K\), with \(J_k := I + \alpha \partial_y f_{W_k}(y_{k-1})\),
\begin{align*}
\Hopnorm{J_k} \le 1+\alpha L_f.
\end{align*}
\end{corollary}
\begin{proof}
By \cref{as2} and the mean value theorem,
\begin{align*}
\Hopnorm{\partial_y f_{W_k}(y_{k-1})} \le L_f.
\end{align*}
Hence
\begin{align*}
\Hopnorm{J_k} \le \Hopnorm{I} + \alpha \Hopnorm{\partial_y f_{W_k}(y_{k-1})} \le 1+\alpha L_f.
\end{align*}
\end{proof}

\begin{corollary}[Jacobian Product Bound]\label{cor:jacprod}
For all \(1\le j\le k\le K\),
\begin{align*}
\Bigl\| J_k J_{k-1} \cdots J_j \Bigr\|_{H\to H} \le (1+\alpha L_f)^{k-j+1} \le e^{\beta L_f} =: C_J.
\end{align*}
\end{corollary}
\begin{proof}
By \cref{cor:jac},
\begin{align*}
\Bigl\| J_k J_{k-1} \cdots J_j \Bigr\|_{H\to H} \le \prod_{i=j}^{k} \Hopnorm{J_i} \le (1+\alpha L_f)^{k-j+1}.
\end{align*}
Substituting \(\alpha = \beta/K\) and using \((1+x/K)^K \le e^x\) for \(x\ge 0\),
\begin{align*}
(1+\alpha L_f)^{k-j+1} \le (1+\beta L_f/K)^K \le e^{\beta L_f}.
\end{align*}
\end{proof}

\begin{corollary}[Derivative Bounds Implied by Assumption 2]\label{cor:deriv}
Under \cref{as2}, the following bounds hold for all \(k=1,\ldots,K\) and all \(y\in H\):
\begin{align}
&\Hopnorm{\partial_y f_{W_k}(y)} \le L_f, \label{eq:bound1}\\
&\Hnorm{\partial_y^2 f_{W_k}(y)[u,v]} \le M_f \Hnorm{u} \Hnorm{v}, \label{eq:bound2}\\
&\PHopnorm{\partial_{W_k} f_{W_k}(y)} \le C_W, \label{eq:bound3}\\
&\Hnorm{\partial_y \partial_{W_k} f_{W_k}(y)[h,w]} \le M_{Wf} \Hnorm{h} \Pnorm{w}, \label{eq:bound4}\\
&\Hnorm{\partial_{W_k}^2 f_{W_k}(y)[u,w]} \le M_{WW} \Pnorm{u} \Pnorm{w}. \label{eq:bound5}
\end{align}
\end{corollary}

\begin{proof}
We prove \eqref{eq:bound1} in detail; the remaining bounds follow by the same argument applied to the corresponding Lipschitz condition in \cref{as2}.

By \cref{as2}, \(f_{W_k}\) is \(L_f\)-Lipschitz in \(y\):
\begin{align*}
\Hnorm{f_{W_k}(y_1)-f_{W_k}(y_2)} \le L_f \Hnorm{y_1-y_2}, \qquad \forall y_1,y_2\in H.
\end{align*}
Fix \(y\in H\) and \(h\in H\). For any \(t>0\), set \(y_1 = y+th\) and \(y_2 = y\). Then
\begin{align*}
\Hnorm{f_{W_k}(y+th)-f_{W_k}(y)} \le L_f \Hnorm{th} = L_f \, t \, \Hnorm{h}.
\end{align*}
Dividing by \(t\) gives
\begin{align*}
\Hnorm{\frac{f_{W_k}(y+th)-f_{W_k}(y)}{t}} \le L_f \Hnorm{h}.
\end{align*}
Taking the limit as \(t\to 0\) and using the definition of the directional derivative,
\begin{align*}
\Hnorm{\partial_y f_{W_k}(y)[h]} \le L_f \Hnorm{h}.
\end{align*}
Taking the supremum over all \(h\in H\) with \(\Hnorm{h}\le 1\),
\begin{align*}
\Hopnorm{\partial_y f_{W_k}(y)}
= \sup_{\Hnorm{h}\le 1} \Hnorm{\partial_y f_{W_k}(y)[h]}
\le L_f.
\end{align*}
This proves \eqref{eq:bound1}.

For \eqref{eq:bound2}, apply the same argument to the \(M_f\)-Lipschitz condition on \(\partial_y f_{W_k}\). For \eqref{eq:bound3}, use the \(C_W\)-Lipschitz condition on \(f_{W_k}\) in \(W_k\). For \eqref{eq:bound4}, use the \(M_{Wf}\)-Lipschitz condition on \(\partial_y f_{W_k}\) in \(W_k\). For \eqref{eq:bound5}, use the \(M_{WW}\)-Lipschitz condition on \(\partial_{W_k} f_{W_k}[\Delta U_k]\) in \(W_k\).
\end{proof}

\subsection{First-Order Directional Derivative Bound}
To simplify our notation, we define:
\begin{definition}[Directional Derivatives]\label{def:directional}
For \(k=0,\ldots,K\) and \(j,\ell=1,\ldots,K\), define
\begin{align*}
s_k^{(j,U_j)} := \nabla_{W_j} y_k [\Delta U_j] \in H,
\end{align*}
\begin{align*}
h_k^{(j,\ell;U_j,V_\ell)} := \nabla^2_{W_j,W_\ell} y_k [\Delta U_j,\Delta V_\ell] \in H.
\end{align*}
\end{definition}

Then for all \(j=1,\ldots,K\), we have:
\begin{align*}
\Hnorm{s_K^{(j,U_j)}} \le \frac{C_s}{K}, \qquad C_s := C_J \beta C_W.
\end{align*}
Consequently,
\begin{align*}
\Bigl\| \sum_{j=1}^{K} s_K^{(j,U_j)} \Bigr\|_H \le C_s.
\end{align*}

\begin{proof}
Differentiating \(y_k = y_{k-1} + \alpha f_{W_k}(y_{k-1})\) with respect to \(W_j\) in direction \(\Delta U_j\),
\begin{align*}
s_k^{(j,U_j)} =
\begin{cases}
J_k s_{k-1}^{(j,U_j)}, & j<k,\\
\alpha \partial_{W_k} f_{W_k}(y_{k-1})[\Delta U_k], & j=k,\\
0, & j>k.
\end{cases}
\end{align*}
Since \(s_0^{(j,U_j)}=0\), we have \(s_j^{(j,U_j)} = \alpha \partial_{W_j} f_{W_j}(y_{j-1})[\Delta U_j]\), and iterating,
\begin{align*}
s_K^{(j,U_j)} = J_K J_{K-1} \cdots J_{j+1} \, \alpha \partial_{W_j} f_{W_j}(y_{j-1})[\Delta U_j].
\end{align*}
By \cref{cor:deriv} (see \eqref{eq:bound3}),
\begin{align*}
\Hnorm{\partial_{W_j} f_{W_j}(y_{j-1})[\Delta U_j]} \le C_W \Pnorm{\Delta U_j} \le C_W.
\end{align*}
By \cref{cor:jacprod},
\begin{align*}
\Hnorm{s_K^{(j,U_j)}} \le C_J \alpha C_W = \frac{C_s}{K}.
\end{align*}
Summing over \(j\),
\begin{align*}
\Bigl\| \sum_{j=1}^{K} s_K^{(j,U_j)} \Bigr\|_H \le \sum_{j=1}^{K} \Hnorm{s_K^{(j,U_j)}} \le K \cdot \frac{C_s}{K} = C_s.
\end{align*}
\end{proof}

\begin{corollary}[Uniform Bound on First-Order Directional Derivatives]
\label{cor:s-general}
From the proof above, it follows immediately that for all \(1\le j\le k\le K\),
\begin{align*}
\Hnorm{s_k^{(j,U_j)}} \le \frac{C_s}{K}, \qquad C_s := C_J \beta C_W.
\end{align*}
In particular, for \(k=K\),
\begin{align*}
\Hnorm{s_K^{(j,U_j)}} \le \frac{C_s}{K}, \qquad
\Bigl\| \sum_{j=1}^{K} s_K^{(j,U_j)} \Bigr\|_H \le C_s.
\end{align*}
\end{corollary}

\subsection{Second-Order Directional Derivative Bound}
For all \(1\le j,\ell\le K\),
\begin{align*}
\Hnorm{h_K^{(j,\ell;U_j,V_\ell)}} \le
\begin{cases}
C_{h1}/K^2, & j\neq \ell,\\
C_{h2}/K, & j=\ell,
\end{cases}
\end{align*}
with \(C_{h1},C_{h2}=\Theta(1)\). Consequently,
\begin{equation}
\Bigl\| \sum_{j=1}^{K} \sum_{\ell=1}^{K} h_K^{(j,\ell;U_j,V_\ell)} \Bigr\|_H \le C_h.
\label{eq:second-order-sum}
\end{equation}

\begin{proof}
We proceed in three stages: derive the recursion, expand to layer \(K\), and estimate.

\paragraph{Stage 1: Recursion.}
Differentiating the recursion for \(s_k^{(j,U_j)}\) with respect to \(W_\ell\) in direction \(\Delta V_\ell\),
\begin{align*}
h_k^{(j,\ell;U_j,V_\ell)}
&= J_k h_{k-1}^{(j,\ell;U_j,V_\ell)}
+ \alpha \partial_y^2 f_{W_k}(y_{k-1})[s_{k-1}^{(j,U_j)}, s_{k-1}^{(\ell,V_\ell)}]\\
&\quad + \alpha \mathbf{1}_{j=k} \partial_y \partial_{W_k} f_{W_k}(y_{k-1})[s_{k-1}^{(\ell,V_\ell)}, \Delta U_k]\\
&\quad + \alpha \mathbf{1}_{\ell=k} \partial_y \partial_{W_k} f_{W_k}(y_{k-1})[s_{k-1}^{(j,U_j)}, \Delta V_k]\\
&\quad + \alpha \mathbf{1}_{j=k} \mathbf{1}_{\ell=k} \partial^2_{W_k} f_{W_k}(y_{k-1})[\Delta U_k, \Delta V_k].
\end{align*}
Since \(h_0^{(j,\ell;U_j,V_\ell)}=0\) and \(s_{k-1}^{(j,U_j)}=0\) for \(j\ge k\), we have:
\begin{itemize}
\item Case A (\(j,\ell<k\)): \(h_k = J_k h_{k-1} + \alpha \partial_y^2 f[s_{k-1}^{(j)}, s_{k-1}^{(\ell)}]\).
\item Case B (\(j=k,\ell<k\)): \(h_k = \alpha \partial_y \partial_{W_k} f[s_{k-1}^{(\ell)}, \Delta U_k]\).
\item Case C (\(\ell=k,j<k\)): \(h_k = \alpha \partial_y \partial_{W_k} f[s_{k-1}^{(j)}, \Delta V_k]\).
\item Case D (\(j=\ell=k\)): \(h_k = \alpha \partial^2_{W_k} f[\Delta U_k, \Delta V_k]\).
\item Case E (\(j>k\) or \(\ell>k\)): \(h_k = 0\).
\end{itemize}

\paragraph{Stage 2: Closed-form expansion.}
We unroll the recursion from the first nonzero layer to layer \(K\).

\emph{Case \(j<\ell\).} The first nonzero term occurs at \(k=\ell\) in Case C:
\begin{align*}
h_\ell^{(j,\ell;U_j,V_\ell)} = \alpha \partial_y \partial_{W_\ell} f_{W_\ell}(y_{\ell-1})[s_{\ell-1}^{(j,U_j)}, \Delta V_\ell].
\end{align*}
For \(k=\ell+1,\ldots,K\), Case A adds the source term
\(\alpha \partial_y^2 f_{W_k}[s_{k-1}^{(j,U_j)}, s_{k-1}^{(\ell,V_\ell)}]\).

Iterating,
\begin{align*}
h_K^{(j,\ell;U_j,V_\ell)}
&= J_K J_{K-1} \cdots J_{\ell+1} \, \alpha \partial_y \partial_{W_\ell} f_{W_\ell}(y_{\ell-1})[s_{\ell-1}^{(j,U_j)}, \Delta V_\ell]\\
&\quad + \sum_{k=\ell+1}^{K} J_K J_{K-1} \cdots J_{k+1} \, \alpha \partial_y^2 f_{W_k}(y_{k-1})[s_{k-1}^{(j,U_j)}, s_{k-1}^{(\ell,V_\ell)}].
\end{align*}

\emph{Case \(j>\ell\).} By symmetry, \(h_K^{(j,\ell;U_j,V_\ell)} = h_K^{(\ell,j;V_\ell,U_j)}\).

\emph{Case \(j=\ell>0\).} The first nonzero term occurs at \(k=j\) in Case D:
\begin{align*}
h_j^{(j,j;U_j,V_j)} = \alpha \partial^2_{W_j} f_{W_j}(y_{j-1})[\Delta U_j, \Delta V_j].
\end{align*}
Iterating,
\begin{align*}
h_K^{(j,j;U_j,V_j)}
&= J_K J_{K-1} \cdots J_{j+1} \, \alpha \partial^2_{W_j} f_{W_j}(y_{j-1})[\Delta U_j, \Delta V_j]\\
&\quad + \sum_{k=j+1}^{K} J_K J_{K-1} \cdots J_{k+1} \, \alpha \partial_y^2 f_{W_k}(y_{k-1})[s_{k-1}^{(j,U_j)}, s_{k-1}^{(j,V_j)}].
\end{align*}

\paragraph{Stage 3: Estimates.}
By \cref{cor:deriv} and the bound on \(s_{k-1}^{(j,U_j)}\) in \cref{cor:s-general}, we have:
\begin{align*}
\Hnorm{\alpha \partial_y^2 f_{W_k}[s_{k-1}^{(j,U_j)}, s_{k-1}^{(\ell,V_\ell)}]}
&\le \alpha M_f \frac{C_s^2}{K^2} = \frac{C_{yy}}{K^3},\\
\Hnorm{\alpha \partial_y \partial_{W_\ell} f_{W_\ell}[s_{\ell-1}^{(j,U_j)}, \Delta V_\ell]}
&\le \alpha M_{Wf} \frac{C_s}{K} = \frac{C_{yW}}{K^2},\\
\Hnorm{\alpha \partial^2_{W_j} f_{W_j}[\Delta U_j, \Delta V_j]}
&\le \alpha M_{WW} = \frac{C_{WW}}{K}.
\end{align*}
By \cref{cor:jacprod}, \(\|J_K \cdots J_{k+1}\|_{H\to H} \le C_J\).

\emph{Case \(j\ne \ell\).} Using the expansion for \(j<\ell\) (the case \(j>\ell\) is symmetric),
\begin{align*}
\Hnorm{h_K^{(j,\ell;U_j,V_\ell)}}
&\le C_J \frac{C_{yW}}{K^2} + \sum_{k=\ell+1}^{K} C_J \frac{C_{yy}}{K^3}
\le \frac{C_J C_{yW}}{K^2} + K \frac{C_J C_{yy}}{K^3}
= \frac{C_J(C_{yW}+C_{yy})}{K^2} =: \frac{C_{h1}}{K^2}.
\end{align*}

\emph{Case \(j=\ell\).} For \(j>0\),
\begin{align*}
\Hnorm{h_K^{(j,j;U_j,V_j)}}
&\le C_J \frac{C_{WW}}{K} + \sum_{k=j+1}^{K} C_J \frac{C_{yy}}{K^3}
\le \frac{C_J C_{WW}}{K} + \frac{C_J C_{yy}}{K^2}
\le \frac{C_J(C_{WW}+C_{yy})}{K} =: \frac{C_{h2}}{K}.
\end{align*}

There are \(K(K-1)\) pairs with \(j\ne\ell\) and \(K\) pairs with \(j=\ell\):
\begin{align*}
\Bigl\| \sum_{j=1}^{K} \sum_{\ell=1}^{K} h_K^{(j,\ell;U_j,V_\ell)} \Bigr\|_H
\le K(K-1) \frac{C_{h1}}{K^2} + K \frac{C_{h2}}{K}
\le C_{h1} + C_{h2} =: C_h.
\end{align*}
\end{proof}

\subsection{Lemmas}
\begin{lemma}[Directional Derivative Lipschitz Bound]\label{lem:lipschitz}
Let \(f:\mathbb{R}^{m\times n}\to\mathbb{R}\) be differentiable and \(\|\cdot\|_B\) be any matrix norm. If for all \(Z\) on the segment between \(Z^1,Z^2\) and all \(\|\Delta U\|_B\le 1\), \(|\nabla f(Z)[\Delta U]|\le L\), then
\begin{align*}
|f(Z^1)-f(Z^2)| \le L \|Z^1-Z^2\|_B.
\end{align*}
\end{lemma}
\begin{proof}
If \(Z^1=Z^2\), the claim is trivial. Assume \(Z^1\neq Z^2\). Let \(Z(t)=Z^2+(Z^1-Z^2)t\) for \(t\in[0,1]\). By the fundamental theorem of calculus,
\begin{align*}
f(Z^1)-f(Z^2) = \int_0^1 \nabla f(Z(t))[Z^1-Z^2] \, dt.
\end{align*}
Set \(\Delta U = (Z^1-Z^2)/\|Z^1-Z^2\|_B\). Then \(\|\Delta U\|_B=1\), and
\begin{align*}
|\nabla f(Z(t))[Z^1-Z^2]| = \|Z^1-Z^2\|_B \, |\nabla f(Z(t))[\Delta U]| \le L \|Z^1-Z^2\|_B.
\end{align*}
Integrating yields the result.
\end{proof}

\begin{lemma}[Directional Hessian Smoothness Bound]\label{lem:smooth}
Let \(f:\mathbb{R}^{m\times n}\to\mathbb{R}\) be twice differentiable and \(\|\cdot\|_B\) be any matrix norm. If for all \(Z\) on the segment between \(Z^1,Z^2\) and all \(\|\Delta U\|_B,\|\Delta V\|_B\le 1\), \(|\nabla^2 f(Z)[\Delta U,\Delta V]|\le L\), then
\begin{align*}
\|\nabla f(Z^1)-\nabla f(Z^2)\|_{B,*} \le L \|Z^1-Z^2\|_B.
\end{align*}
\end{lemma}
\begin{proof}
If \(Z^1=Z^2\), the claim is trivial. Assume \(Z^1\neq Z^2\). For any direction \(\Delta U\),
\begin{align*}
\nabla f(Z^1)[\Delta U] - \nabla f(Z^2)[\Delta U] = \int_0^1 \nabla^2 f(Z(t))[\Delta U, Z^1-Z^2] \, dt.
\end{align*}
By the definition of the dual norm,
\begin{align*}
\|\nabla f(Z^1)-\nabla f(Z^2)\|_{B,*} = \sup_{\|\Delta U\|_B\le 1} |\nabla f(Z^1)[\Delta U] - \nabla f(Z^2)[\Delta U]|.
\end{align*}
Set \(\Delta V = (Z^1-Z^2)/\|Z^1-Z^2\|_B\). Then \(\|\Delta V\|_B=1\), and
\begin{align*}
|\nabla^2 f(Z(t))[\Delta U, Z^1-Z^2]| = \|Z^1-Z^2\|_B \, |\nabla^2 f(Z(t))[\Delta U, \Delta V]| \le L \|Z^1-Z^2\|_B.
\end{align*}
Integrating and taking the supremum yields the result.
\end{proof}

\subsection{Proof of the Theorem}
\begin{theorem*}[Depth-Independent Lipschitz Continuity]
Under \cref{as1,as2,as3} and \(\alpha=\beta/K\), \(\beta=\Theta(1)\), there exists \(L_{\mathrm{Lip}}=O(1)\), independent of \(K\), such that
\begin{align*}
\sup_{\|\Delta U\|_{\mathrm{block}}\le 1} |\nabla \mathcal{L}(W)[\Delta U]| \le L_{\mathrm{Lip}}.
\end{align*}
Consequently,
\begin{align*}
|\mathcal{L}(W^1)-\mathcal{L}(W^2)| \le L_{\mathrm{Lip}} \|W^1-W^2\|_{\mathrm{block}}.
\end{align*}
\end{theorem*}
\begin{proof}
By the chain rule,
\begin{align*}
\nabla_W \mathcal{L}(W)[\Delta U] = l'(y_K) \sum_{j=1}^{K} s_K^{(j,U_j)}.
\end{align*}
By \cref{as3} and \cref{cor:s-general},
\begin{align*}
|\nabla_W \mathcal{L}(W)[\Delta U]| \le L_J C_s =: L_{\mathrm{Lip}}.
\end{align*}
\Cref{lem:lipschitz} gives the Lipschitz consequence.
\end{proof}

\begin{theorem*}[Depth-Independent \(L\)-Smoothness]
Under \cref{as1,as2,as3} and \(\alpha=\beta/K\), \(\beta=\Theta(1)\), there exists \(L=O(1)\), independent of \(K\), such that
\begin{align*}
\sup_{\|\Delta U\|_{\mathrm{block}}\le 1, \|\Delta V\|_{\mathrm{block}}\le 1} |\nabla^2 \mathcal{L}(W)[\Delta U,\Delta V]| \le L.
\end{align*}
Consequently,
\begin{align*}
\|\nabla \mathcal{L}(W^1)-\nabla \mathcal{L}(W^2)\|_{\mathrm{block},*} \le L \|W^1-W^2\|_{\mathrm{block}}.
\end{align*}
\end{theorem*}
\begin{proof}
By the chain rule,
\begin{align*}
\nabla_W^2 \mathcal{L}[\Delta U,\Delta V]
= l''(y_K) \Bigl( \sum_{j=1}^{K} s_K^{(j,U_j)} \Bigr) \Bigl( \sum_{\ell=1}^{K} s_K^{(\ell,V_\ell)} \Bigr)
+ l'(y_K) \sum_{j=1}^{K} \sum_{\ell=1}^{K} h_K^{(j,\ell;U_j,V_\ell)}.
\end{align*}
By \cref{as3}, \cref{cor:s-general}, and \eqref{eq:second-order-sum},
\begin{align*}
|\nabla_W^2 \mathcal{L}[\Delta U,\Delta V]| \le M_J C_s^2 + L_J C_h =: L.
\end{align*}
\Cref{lem:smooth} gives the smoothness consequence.
\end{proof}

\section{Discussion}

In this section we verify that the hypotheses of \cref{thm:lipschitz} and \cref{thm:smooth} ---
in particular, the residual-function regularity of \cref{as2} --- are met by the building blocks
used in practice. We first treat multi-layer perceptron (MLP) residual blocks and record exactly which part of \cref{as2} needs to hold only locally; we also show that attention blocks satisfy the same regularity at fixed sequence length and feature dimensions. Throughout,
$H=\mathbb{R}^d$ is equipped with the $(p,\mathrm{mean})$ norm $\Hnorm{x}=d^{-1/p}\,\|x\|_p$ for
some $1\le p\le\infty$, and $\Hopnorm{\cdot}$ denotes the $(p,\mathrm{mean})\to(p,\mathrm{mean})$
operator norm. We use $\|\cdot\|_{H\rightarrow H}$ for the induced operator norm and for the parameter space $P$ we measure matrix components with $\|\cdot\|_{(p,\mathrm{mean})\rightarrow(q,\mathrm{mean})}$ for the prescribed matrix norm from $(p,\mathrm{mean})$ to $(q,\mathrm{mean})$ where $p\leq q$. By mean normalization, we have
\begin{align*}
    \|A\|_{H \rightarrow H}
    \leq \|A\|_{(p,\mathrm{mean})\rightarrow(q,\mathrm{mean})},
    \qquad A\in\mathbb{R}^{d\times d}.
\end{align*}
Thus, bounds in the prescribed matrix norm also control hidden-state propagation. Moreover,
\[
    \|\operatorname{Diag}(v)\|_{H\to H}=\|v\|_\infty,
    \qquad
    \|\mathbf{1}\|_H=1,
\]
where \(\mathbf{1}=(1,\ldots,1)^\top\).
The hidden dimension \(d\) and the norm exponents \(p,q\)
are fixed independently of the residual depth \(K\).

\subsection{Multi-Layer Neural Networks}\label{subsec:mlp}

\begin{remark}[Localization of the parameter-regularity conditions]\label{rem:loc}
In the proofs of \cref{thm:lipschitz} and \cref{thm:smooth}, the $W$-regularity conditions of
\cref{as2} --- namely the bounds \eqref{eq:bound3}--\eqref{eq:bound5} on $\partial_{W_k}f_{W_k}$,
$\partial_y\partial_{W_k}f_{W_k}$, and $\partial^2_{W_k}f_{W_k}$ --- are invoked only at the
iterates $y_{k-1}$, $k=1,\ldots,K$. By \cref{cor:fwd}, every iterate satisfies
$\Hnorm{y_{k-1}}\le C_y$. Consequently, it suffices for the theorems that these three conditions
hold only for $y$ in the ball $\{\Hnorm{y}\le C_y\}$, rather than for all $y\in H$. The
$y$-regularity conditions (the first two lines of \cref{as2}, plus the bound
$\Hnorm{f_{W_k}(y_0)}\le C_f$) are genuinely global and are unaffected.
\end{remark}

\begin{proposition}[MLP residual satisfies Assumption 2]\label{prop:mlp}
Let $H=\mathbb{R}^d$ and let the residual function be the $N$-layer MLP
\begin{align*}
f_{W_k}(y) = \sigma\bigl( W^{(N)} \sigma( \cdots \sigma(W^{(1)} y + b^{(1)}) \cdots ) + b^{(N)} \bigr),
\end{align*}
where $W_k=(W^{(1)},\ldots,W^{(N)},b^{(1)},\ldots,b^{(N)})$ with
$W^{(1)},\ldots,W^{(N)}\in\mathbb{R}^{d\times d}$ and $b^{(1)},\ldots,b^{(N)}\in\mathbb{R}^d$,
and equip $P$ with the norm
\begin{align*}
\Pnorm{W_k} = \max\bigl\{\|W^{(1)}\|_{(p,\mathrm{mean})\rightarrow(q,\mathrm{mean})},\ldots,\|W^{(N)}\|_{(p,\mathrm{mean})\rightarrow(q,\mathrm{mean})},\Hnorm{b^{(1)}},\ldots,\Hnorm{b^{(N)}}\bigr\}.
\end{align*}
Suppose that
\begin{enumerate}
\item $\sigma$ is twice continuously differentiable, with constants chosen so that
      \[
      \max\{1,|\sigma(0)|,\sup_{t\in\mathbb{R}}|\sigma'(t)|\}
      \le L_\sigma,
      \qquad
      d^{1/p}\sup_{t\in\mathbb{R}}|\sigma''(t)|
      \le M_\sigma,
      \]
      where $d^{1/\infty}=1$;
\item the internal depth $N$ is fixed and does not grow with $K$;
\item for all $\ell=1,\ldots,N$, $\|W^{(\ell)}\|_{(p,\mathrm{mean})\rightarrow(q,\mathrm{mean})}\le B_W$ and $\Hnorm{b^{(\ell)}}\le B_b$,
      with $B_W,B_b$ independent of $K$.
\end{enumerate}
Without loss of generality, enlarge the weight bound so that $B_W\ge1$.
The admissible parameter domain is the convex set
\[
\Omega=\left\{
W_k\in P:
\|W^{(\ell)}\|_{(p,\mathrm{mean})\rightarrow(q,\mathrm{mean})}\le B_W,\ 
\Hnorm{b^{(\ell)}}\le B_b,\quad \ell=1,\ldots,N
\right\}.
\]
All parameter-regularity comparisons are restricted to $\Omega$.

Then $f_{W_k}$ satisfies \cref{as2} in the localized sense of \cref{rem:loc}: the $y$-regularity
conditions hold globally, while the $W$-regularity conditions hold on the ball
$\{\Hnorm{y}\le C_y\}$, all with constants $L_f,M_f,C_W,M_{Wf},M_{WW},C_f$ depending only on
$N,L_\sigma,M_\sigma,B_W,B_b,C_x,C_y$, and not on $K$.
\end{proposition}

\begin{proof}
\emph{Forward pass.} Introduce the pre-activations $z^{(\ell)}$ and activations $a^{(\ell)}$ by
\begin{align*}
a^{(0)}=y,\qquad
z^{(\ell)}=W^{(\ell)}a^{(\ell-1)}+b^{(\ell)},\qquad
a^{(\ell)}=\sigma(z^{(\ell)}),\qquad f_{W_k}(y)=a^{(N)}.
\end{align*}
By mean-normalization we have that $\|W\|_{H\rightarrow H} \leq \|W\|_{(p,mean)\rightarrow(q,mean)}\leq B_W$.Write $D^{(\ell)}:=\mathrm{Diag}(\sigma'(z^{(\ell)}))$ and $G^{(\ell)}:=D^{(\ell)}W^{(\ell)}$ for
the per-layer Jacobian, so that
\begin{align*}
\Hopnorm{D^{(\ell)}}\le L_\sigma,\qquad \Hopnorm{G^{(\ell)}}\le L_\sigma B_W
\end{align*}
hold globally (for all $y$). Let $F^{(\ell)}:=G^{(N)}G^{(N-1)}\cdots G^{(\ell+1)}$ denote forward
propagation from layer $\ell$ to the output (with $F^{(N)}=I$), so that
\begin{align*}
\Hopnorm{F^{(\ell)}}\le (L_\sigma B_W)^{N-\ell}.
\end{align*}

\emph{Bounded activations.} By the mean value theorem applied componentwise to
$a^{(\ell)}=\sigma(z^{(\ell)})$, there exists a diagonal matrix
$\tilde{D}^{(\ell)}=\mathrm{Diag}(\sigma'(\xi_1),\ldots,\sigma'(\xi_d))$ with
$\Hopnorm{\tilde{D}^{(\ell)}}\le L_\sigma$ such that
$a^{(\ell)}=\sigma(0)\mathbf{1}+\tilde{D}^{(\ell)}z^{(\ell)}$. Since $\Hnorm{\mathbf{1}}=1$, the
triangle inequality gives
\begin{align*}
\Hnorm{a^{(\ell)}} \le |\sigma(0)| + L_\sigma\Hnorm{z^{(\ell)}}.
\end{align*}
Combining this with $\Hnorm{z^{(\ell)}}\le B_W\Hnorm{a^{(\ell-1)}}+B_b$ and
$\Hnorm{a^{(0)}}=\Hnorm{y}\le C_y$, an induction over the $N=O(1)$ layers yields a constant
$R=R(N,L_\sigma,B_W,B_b,C_y)\ge1$ such that
\begin{align*}
\Hnorm{a^{(\ell)}}&\le R,\qquad \ell=0,\ldots,N,\\
\Hnorm{z^{(\ell)}}&\le R,\qquad \ell=1,\ldots,N,
\qquad\text{on } \{\Hnorm{y}\le C_y\}.
\end{align*}

\emph{(i) Global $y$-regularity.} The chain rule gives
\begin{align*}
\partial_y f_{W_k}(y) = G^{(N)}G^{(N-1)}\cdots G^{(1)},
\end{align*}
hence
\begin{align*}
\Hopnorm{\partial_y f_{W_k}(y)} \le (L_\sigma B_W)^N =: L_f
\qquad\text{for all } y\in H,
\end{align*}
which is \eqref{eq:bound1}; the $L_f$-Lipschitz condition in $y$ follows by the mean value theorem.

For the second derivative bound, expand the difference of the Jacobian products
as a telescoping sum: for $y_1,y_2\in H$,
\begin{align*}
\partial_y f(y_1)-\partial_y f(y_2)
= \sum_{\ell=1}^{N} G^{(N)}(y_1)\cdots G^{(\ell+1)}(y_1)\,
  \bigl[G^{(\ell)}(y_1)-G^{(\ell)}(y_2)\bigr]\,
  G^{(\ell-1)}(y_2)\cdots G^{(1)}(y_2).
\end{align*}
Since $G^{(\ell)}=D^{(\ell)}W^{(\ell)}$ and $\partial_y z^{(\ell)}=W^{(\ell)}G^{(\ell-1)}\cdots G^{(1)}$,
each factor is Lipschitz in $y$:
\begin{align*}
\Hopnorm{G^{(\ell)}(y_1)-G^{(\ell)}(y_2)}
&\le \Hopnorm{D^{(\ell)}(y_1)-D^{(\ell)}(y_2)}\,\Hopnorm{W^{(\ell)}}\\
&\le M_\sigma\,\Hnorm{z^{(\ell)}(y_1)-z^{(\ell)}(y_2)}\,B_W\\
&\le M_\sigma B_W^2(L_\sigma B_W)^{\ell-1}\Hnorm{y_1-y_2}.
\end{align*}
Substituting into the Leibniz sum (each term carries $N-1$ remaining factors of norm at most
$L_\sigma B_W$) yields $\Hopnorm{\partial_y f(y_1)-\partial_y f(y_2)}\le M_f\Hnorm{y_1-y_2}$ with
\begin{align*}
M_f := M_\sigma B_W^2 \sum_{\ell=1}^{N}(L_\sigma B_W)^{N+\ell-2},
\end{align*}
a constant independent of $K$; equivalently \eqref{eq:bound2} holds. This verifies the two global
$y$-conditions of \cref{as2}.

\emph{(ii) Localized $W$-regularity on the ball $\{\Hnorm{y}\le C_y\}$.} Differentiating through
layer $\ell$, the derivative with respect to the weight and bias blocks is
\begin{align*}
\partial_{W^{(\ell)}}f_{W_k}(y)[\Delta W^{(\ell)}]
  &= F^{(\ell)} D^{(\ell)}\,\Delta W^{(\ell)} a^{(\ell-1)},\\
\partial_{b^{(\ell)}}f_{W_k}(y)[\Delta b^{(\ell)}]
  &= F^{(\ell)} D^{(\ell)}\,\Delta b^{(\ell)}.
\end{align*}
Using $\Hopnorm{F^{(\ell)}}\le (L_\sigma B_W)^{N-\ell}$, $\Hopnorm{D^{(\ell)}}\le L_\sigma$,
$\Hnorm{a^{(\ell-1)}}\le R$, and $\Pnorm{\Delta W_k}\le 1$ (so each block has norm at most $1$),
we get
\begin{align*}
\PHopnorm{\partial_{W_k}f_{W_k}(y)}
\le \sum_{\ell=1}^{N}(L_\sigma B_W)^{N-\ell}L_\sigma(R+1) =: C_W,
\end{align*}
which is \eqref{eq:bound3}; the mean value theorem over the convex parameter set
$\Omega$ gives the $C_W$-Lipschitz condition in $W$.

For \eqref{eq:bound4}, differentiate $\partial_y f=G^{(N)}\cdots G^{(1)}$ with respect to
$W^{(\ell)}$ (for $m\ge\ell$, $G^{(m)}$ depends on $W^{(\ell)}$ through $z^{(m)}$; $G^{(\ell)}$
depends additionally on $W^{(\ell)}$ directly):
\begin{align*}
\partial_{W^{(\ell)}}\partial_y f\,[\Delta W^{(\ell)}]
= \sum_{m=\ell}^{N} G^{(N)}\cdots G^{(m+1)}\,
  \partial_{W^{(\ell)}}G^{(m)}[\Delta W^{(\ell)}]\,
  G^{(m-1)}\cdots G^{(1)}.
\end{align*}
We bound the differentiated factor on the ball, taking
$\|\Delta W^{(\ell)}\|_{(p,\mathrm{mean})\to(q,\mathrm{mean})}\le1$.
For $m=\ell$, differentiating
$G^{(\ell)}=D^{(\ell)}W^{(\ell)}$ gives
\begin{align*}
\Hopnorm{\partial_{W^{(\ell)}}G^{(\ell)}[\Delta W^{(\ell)}]}
&\le \Hopnorm{\partial_{W^{(\ell)}}D^{(\ell)}[\Delta W^{(\ell)}]}\,\Hopnorm{W^{(\ell)}}
  +\Hopnorm{D^{(\ell)}}\,\Hopnorm{\Delta W^{(\ell)}}\\
&\le M_\sigma R\,B_W + L_\sigma,
\end{align*}
where the two terms come from differentiating $D^{(\ell)}$ and $W^{(\ell)}$, respectively. For
$m>\ell$, only the $D^{(m)}$-factor depends on $W^{(\ell)}$, through
\begin{align*}
\partial_{W^{(\ell)}}z^{(m)}[\Delta W^{(\ell)}]
= W^{(m)}G^{(m-1)}\cdots G^{(\ell+1)}
D^{(\ell)}\Delta W^{(\ell)}a^{(\ell-1)},
\end{align*}
with the intermediate product interpreted as the identity when $m=\ell+1$,
so that
\begin{align*}
\Hopnorm{\partial_{W^{(\ell)}}G^{(m)}[\Delta W^{(\ell)}]}
&\le M_\sigma\,\Hnorm{\partial_{W^{(\ell)}}z^{(m)}[\Delta W^{(\ell)}]}\,\Hopnorm{W^{(m)}}\\
&\le M_\sigma B_W (L_\sigma B_W)^{m-\ell} R.
\end{align*}
Setting
\begin{align*}
\Gamma := \max\bigl\{L_\sigma B_W,\ L_\sigma+M_\sigma R B_W,\ M_\sigma B_W (L_\sigma B_W)^{N-1} R\bigr\},
\end{align*}
each summand above is a product of $N$ factors of norm at most $\Gamma$.
The same bounds apply to unit bias directions, replacing
$\Delta W^{(\ell)}a^{(\ell-1)}$ by $\Delta b^{(\ell)}$ and omitting the direct
weight-derivative term, since $R\ge1$.
Summing the at most $N$ summands and accounting for the $N$ weight and $N$ bias blocks gives
\begin{align*}
\left\lVert \partial_y\partial_{W_k}f_{W_k}(y)\right\rVert_{H\times P\to H}
\le 2N^2\Gamma^{N} =: M_{Wf};
\end{align*}
the mean value theorem on the hidden-state ball yields \eqref{eq:bound4}.

\emph{(iii) Second-order $W$-regularity.} The map $\partial_{W_k}f_{W_k}(y)$ is the sum over
$\ell=1,\ldots,N$ of the two blocks displayed in \emph{(ii)}. Differentiating each block once more
with respect to $W_k$ applies the product rule a second time, so every resulting term is a product
of factors, each of which is one of $\sigma'(z^{(m)})$, $\sigma''(z^{(m)})$, a weight matrix
$W^{(m)}$, an activation $a^{(m)}$, the input $y$, or a direction $\Delta U_k,\Delta V_k$.
Take $\Pnorm{\Delta U_k},\Pnorm{\Delta V_k}\le1$.
By the choice of $M_\sigma$, the second-order activation terms satisfy
\begin{align*}
\Hnorm{\sigma''(z)\odot u\odot v}
\le d^{1/p}\sup_{t\in\mathbb{R}}|\sigma''(t)|\,\Hnorm{u}\,\Hnorm{v}
\le M_\sigma\,\Hnorm{u}\,\Hnorm{v},
\end{align*}
where $\odot$ denotes componentwise multiplication.
On the ball $\{\Hnorm{y}\le C_y\}$ every such factor is bounded,
and the total number of terms is bounded by a constant depending only on $N$. Hence
\begin{align*}
\left\lVert \partial^2_{W_k}f_{W_k}(y)\right\rVert_{P\times P\to H} \le M_{WW}
\qquad\text{on } \{\Hnorm{y}\le C_y\},
\end{align*}
and the mean value theorem over $\Omega$ gives \eqref{eq:bound5}. All constants depend only on
$N,L_\sigma,M_\sigma,B_W,B_b,C_y$.

\emph{(iv) The initial bound.} Since $\Hnorm{y_0}=\Hnorm{x}\le C_x$, rerunning the
bounded-activation induction on the ball $\{\Hnorm{y}\le C_x\}$ (in place of $C_y$) gives
\begin{align*}
\Hnorm{f_{W_k}(y_0)} = \Hnorm{a^{(N)}} \le R_x =: C_f,
\qquad R_x = R(N,L_\sigma,B_W,B_b,C_x),
\end{align*}
which does not depend on $C_y$. This completes the verification of \cref{as2} in the localized
sense of \cref{rem:loc}.
\end{proof}

\begin{corollary}[MLP residual networks satisfy the main theorems]\label{cor:mlp-thm}
Let the residual function of \cref{def:resnet} be the MLP of \cref{prop:mlp}, and suppose
\cref{as1} and \cref{as3} hold (bounded input $x$; bounded, smooth loss $l$). Then the loss
$\mathcal{L}$ satisfies the conclusions of \cref{thm:lipschitz} and \cref{thm:smooth}: there exist
constants $L_{\mathrm{Lip}},L=O(1)$, independent of the depth $K$, such that $\mathcal{L}$ is
$L_{\mathrm{Lip}}$-Lipschitz and $L$-smooth in the block norm.
\end{corollary}

\begin{remark}[Scope of the remaining assumptions]\label{rem:remaining}
\Cref{as1} is a data-normalization hypothesis (bounded inputs and unit perturbation directions,
the latter by homogeneity). \Cref{as3} is a hypothesis on the loss function $l$ only, not on the
network, and is satisfied by standard choices (e.g.\ logistic or quadratic losses on a bounded
domain). Thus the only network-dependent hypothesis --- \cref{as2} --- is exactly what
\cref{prop:mlp} establishes for MLP residual blocks.
\end{remark}

\subsection{Attention Blocks}\label{subsec:attention}

\Cref{prop:mlp} verifies \cref{as2} for MLP residual blocks equipped with the $p\to q$ operator norm on
$H=\mathbb{R}^d$. A natural question is whether \cref{as2} --- and therefore the depth-independent guarantees of  \cref{thm:lipschitz,thm:smooth} --- also holds for attention blocks.
\citep{xu2026moga} answer this in the affirmative from the matrix-operator-norm perspective, which
we summarize here.

A single attention head with query, key, value, and output projections $W_q,W_k,W_v,W_o$ maps a
sequence of token features $X\in\mathbb{R}^{s\times d}$ to
\begin{align*}
\mathrm{Att}(X) = \mathrm{softmax}\!\Bigl(\frac{X W_q W_k^{\top} X^{\top}}{\sqrt{d_v}}\Bigr)\, X W_v W_o,
\end{align*}
where $s$ is the sequence length and $d_v$ the per-head value dimension. The regularity of this map
is governed by the scale of the attention logits $XW_qW_k^{\top}X^{\top}$: to keep the softmax
$L$-smooth, this scale must not grow with the width. \citep{xu2026moga} control it by row or column
normalization of the weight matrices.

\emph{Row normalization} (geometry $(p,\mathrm{mean})\to\ell_\infty$). The query and key vectors are
bounded in $\ell_\infty$, so each query--key inner product satisfies
\begin{align*}
|q^{\top}k| \le d_v\,\|q\|_\infty\|k\|_\infty = O(d_v),
\end{align*}
and scaling the logits by $1/d_v$ keeps them $O(1)$ in the width.

\emph{Column normalization} (geometry $(1,\mathrm{mean})\to(q,\mathrm{mean})$). The inner product
scales as $q^{\top}k = O\bigl(d_v^{\max\{1,2/q\}}\bigr)$, so scaling the logits by
$d_v^{-\max\{1,2/q\}}$ keeps them width-independent.

With the logits thus controlled, the loss of a
$K$-layer network composed of such blocks is Lipschitz and smooth with constants independent of the
hidden width, under bounded inputs and parameters and bounded first/second derivatives of the
activation and loss. This is precisely the residual-function regularity demanded by \cref{as2},
expressed in the mean-normalized operator-norm geometry. We record this as follows.

\begin{proposition}[Attention blocks satisfy the residual regularity]\label{prop:attention}
Under row or column normalization of $W_q,W_k,W_v,W_o$ and bounded activations, the attention block
satisfies \cref{as2}
\end{proposition}

\begin{proof}[Proof sketch]
Row normalization bounds the query and key vectors in $\ell_\infty$, so the logits are $O(d_v)$ and
the $1/d_v$ scaling makes the argument of the softmax $O(1)$. This scaling controls the magnitude of the logits. To establish regularity of the full attention branch, we also account for its dependence on the hidden states and projection parameters. Column normalization bounds the query and key vectors in the
$(q,\mathrm{mean})$ norm, giving logits of size $O(d_v^{\max\{1,2/q\}})$, which is compensated by the
$d_v^{-\max\{1,2/q\}}$ scaling. For fixed sequence length and feature dimensions, the attention branch is jointly twice continuously differentiable in its input and projection parameters. On a bounded closed input domain and a bounded closed parameter domain, its first and second derivatives are therefore uniformly bounded in the prescribed norms. In particular, this bounds the state, parameter, and mixed derivatives required by \cref{as2}. If the architecture and domain bounds are independent of the residual depth $K$, the resulting regularity constants are independent of $K$. On convex domains, the corresponding Lipschitz inequalities follow by integration along line segments.
\end{proof}

Thus, the MLP and attention branches satisfy the required regularity conditions on their respective admissible domains. The depth-independent conclusions of \cref{thm:lipschitz,thm:smooth} apply provided the regularity bounds hold uniformly on a domain containing the hidden-state trajectories for all admissible parameter tuples, with parameter comparisons restricted to segments within the admissible parameter domain.

\newpage
\section{Additional Experiments}
Here we provide the validation loss curves of GPT-2 and the validation loss curves of ResNet. 
\subsection{GPT2 Val Loss}
\begin{figure}[H]
    \centering
    \includegraphics[width=\linewidth]
        {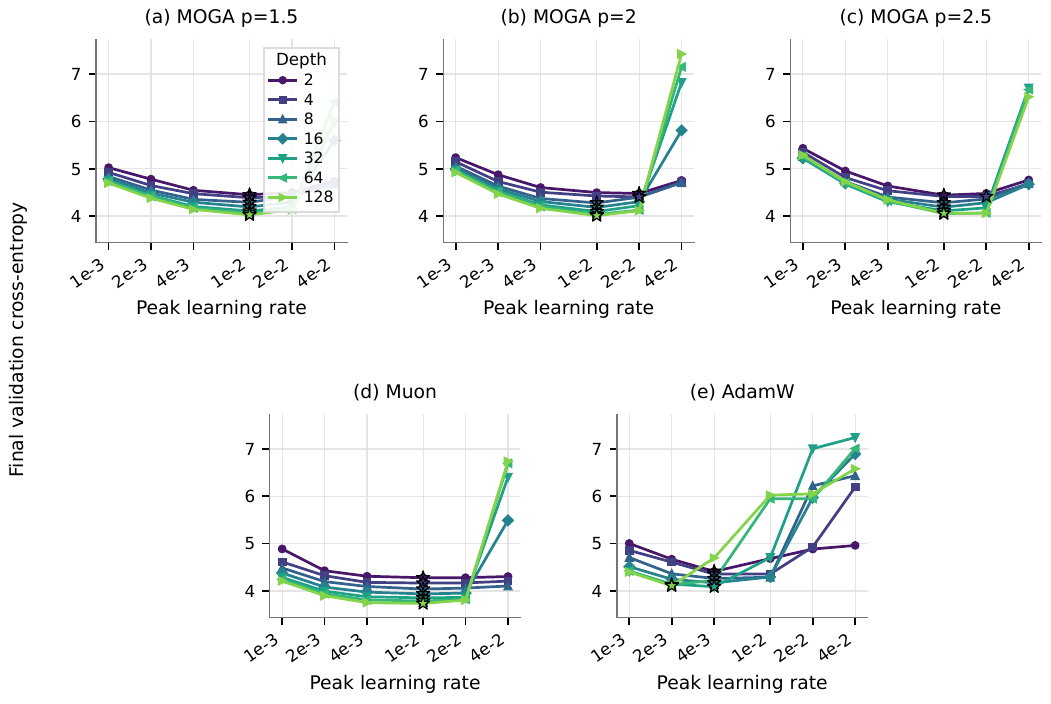}
    \caption{
        Learning-rate transfer across GPT-2 depths at fixed width.
        Panels show MOGA with $p\in\{1.5,2,2.5\}$, Muon, and AdamW.
        Colors and markers indicate transformer depth.
        Stars mark the lowest validation loss among the tested
        learning rates at each depth.
    }
    \label{fig:transformer-lr-transfer-2}
\end{figure}
As shown in Figure~\ref{fig:transformer-lr-transfer-2}, models that vary drastically in depth obtain their best validation loss performance at around the same peak learning rate. Muon and MOGA show strong learning rate transfer as all depths select nearly the same learning rate. The 4 Muon/MOGA panels show a very stable low loss region and a learning rate at a smaller depth can generally be transferred to a deeper model. AdamW exhibits the same instability for deeper models at the optimal learning rate $\eta^*=4 \times 10^{-3}$ of the shallow models and the optimal learning rate shifts slightly towards the left. The optimal learning rates across 7 depths still stay within a tight region.

\subsection{ResNet Val-loss}
\begin{figure}[H]
    \centering
    \includegraphics[width=\linewidth]{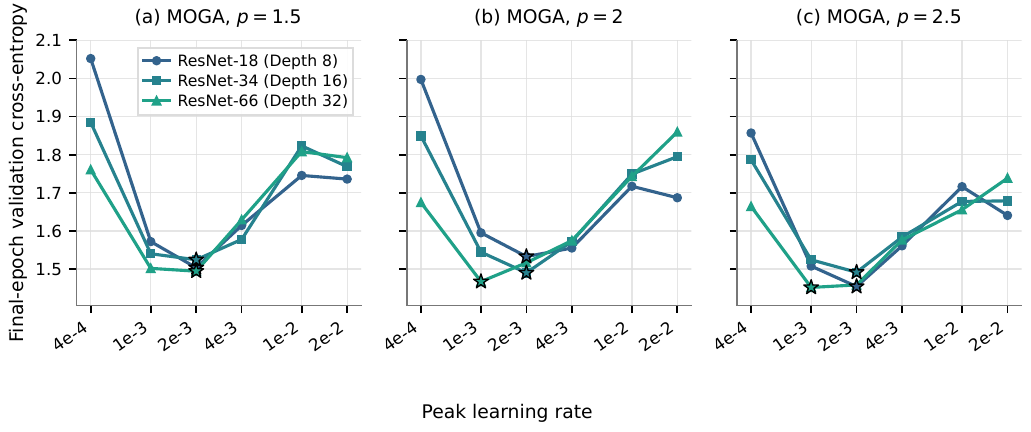}
    \caption{
        Learning-rate transfer across ResNet depths.
        Final validation loss is plotted against the peak learning rate,
        with each curve corresponding to a different depth. Colors and markers indicate ResNet depth.
        Stars mark the lowest validation loss among the tested
        learning rates at each depth.
    }
    \label{fig:resnet_val_loss-2}
\end{figure}
The ResNet Validation loss curves shown in figure~\ref{fig:resnet_val_loss-2} show that the optimal learning rates shifts slightly towards the left for ResNet-66 ($K=32$). MOGA $p=1.5$ and MOGA $p=2.5$ show nearly identical validation loss for $\eta =1 \times 10^{-3}$ and $\eta=2 \times 10^{-3}$, therefore the regret of selecting the optimal learning rate from shallower ResNets and transferring to ResNet-66 is low. However, MOGA $p=2$ clearly shows a preference for the smaller learning rate $\eta =1\times 10^{-3}$. Furthermore, increasing depth does not consistently improve validation loss at a fixed learning rate.

The above shows that the optimal learning rate still stays in a very tight region, but learning rate transfer is not as stable as evaluated on training loss. The models exhibited overfitting during training, with improvements in training loss not consistently translating into improvements in validation loss. Accordingly, learning-rate transfer measured by training loss should be distinguished from transfer measured by validation loss.

\end{document}

%% file: math_commands.tex
\usepackage{amsmath,amsfonts,bm}

\def\eqref#1{equation~\ref{#1}}

\def\1{\bm{1}}

\DeclareMathAlphabet{\mathsfit}{\encodingdefault}{\sfdefault}{m}{sl}
\SetMathAlphabet{\mathsfit}{bold}{\encodingdefault}{\sfdefault}{bx}{n}

\DeclareMathOperator*{\argmin}{arg\,min}